\documentclass{article}
\usepackage{amssymb,amsmath}
\usepackage{algorithm}
\usepackage{algorithmic}
\usepackage{enumitem}
\usepackage{amsthm}
\usepackage{amsfonts}
\usepackage{subfig}
\usepackage{graphicx}
\usepackage{booktabs}
\usepackage{bbm}
\usepackage{algorithm}
\usepackage{algorithmic}
\usepackage{multirow}
\usepackage{url}

\newtheorem{thm}{Theorem}
\newtheorem{mydef}{Definition}
\newtheorem{lemma}{Lemma}
\newcommand {\x}{\mathbf {x}}
\newcommand {\bx}{\mathbf {x}}
\newcommand {\I}{\mathbf {I}}
\newcommand {\bI}{\mathbf {I}}

\newcommand {\bP}{\mathbf {P}}
\newcommand {\bB}{\mathbf {B}}
\newcommand {\bu}{\mathbf {u}}
\newcommand {\R}{\mathbb {R}}
\newcommand {\reals}{\mathbb {R}}

\newcommand{\cB}{\mathcal{B}}

\newcommand{\bbP}{\mathbb{P}}
\newcommand{\bbE}{\mathbb{E}}
\newcommand{\cX}{\mathcal{X}}
\newcommand{\cY}{\mathcal{Y}}
\newcommand{\cA}{\mathcal{A}}
\DeclareMathOperator*{\argmax}{arg\,max}
\DeclareMathOperator*{\argmin}{arg\,min}

\newcommand{\bes}{\begin{equation*}}
\newcommand{\ees}{\end{equation*}}

\newcommand{\be}{\begin{equation}}
\newcommand{\ee}{\end{equation}}

\newcommand{\bi}{\begin{itemize}}
\newcommand{\ei}{\end{itemize}}

\begin{document}
\title{Consistency Analysis of Nearest Subspace Classifier}
\date{August 18, 2014}
\author{{Yi (Grace) Wang},\\[2pt]
Department of Mathematics, Duke University, Durham, NC, USA \\
yiwang@math.duke.edu}

\maketitle

\begin{abstract}
{The Nearest subspace classifier (NSS) finds an estimation of the underlying subspace within each class and assigns data points to the class that corresponds to its nearest subspace. This paper mainly studies how well NSS can be generalized to new samples. It is proved that NSS is strongly consistent under certain assumptions. For completeness, NSS is evaluated through experiments on various simulated and real data sets, in comparison with some other linear model based classifiers. It is also shown that NSS can obtain effective classification results and is very efficient, especially for large scale data sets.}
{Nearest Subspace, Classification, Consistency, Unsupervised Learning}
\end{abstract}

%
%
%
%

\section{Introduction \label{sec:intro}}
The problem of classification is to construct a mapping that can correctly predict the classes of new objects, given training examples of old objects with ground truth labels~\cite{MachineLearning97}. It is a classical problem in statistical learning and machine learning and has been widely used in computer vision, pattern recognition, bioinformatics, etc. Examples of applications include face recognition, handwriting recognition and micro-array classification.

More precisely, this problem can be formalized as follows. Given a training data set $\{(\x_i,y_i)\}_{i=1}^n$, where $\x_i \in \cX$ and $y_i \in \cY$, the goal is to find a function $f:\cX \rightarrow \cY$ such that $f(\x)$ is a good approximation of $y$ for the given $\x_i$'s as well as for new instances $\x$. Typically, $\cX$ is a continuous domain and $\cY$ is a finite discrete set.

In the past few decades, a tremendous amount of work has been produced for this problem. Many approaches have been proposed, e.g., K-Nearest Neighbors (KNN)~\cite{FixHodges51,Cover06,Duda73}, Fisher's Linear Discriminant Analysis (LDA)~\cite{fisher36lda,rao1952advanced}, Artificial Neural Networks (ANN)~\cite{rosenblatt58a,Werbos74,Minsky88}, Support Vector Machines (SVM)~\cite{Boser92,CortesVapnik93,sch95}, and Decision Trees (see~\cite{cart84,Quinlan79,Quinlan93} for some well known algorithms). We refer to~\cite{ElemStat01,Bishop06} for a more careful overview of classification techniques.

Among this work is a class of methods based on subspace models. The compelling interest in subspace models can be attributed to their validation in real data. For instance, it has been justified that the set of all images of a Lambertian object (e.g., face images) under a variety of lighting conditions can be accurately approximated by a low-dimensional linear subspace (of dimension at most 9)~\cite{bb34497,Ho03,Basri03}. Another example is that, under the affine camera model, the coordinate vectors of feature points from a moving rigid object lie in an affine subspace of dimension at most 3 (see~\cite{Costeira98}). These applications give rise to modeling data by subspaces; the study of subspace based classifiers is an important branch.

The first work in this category was CLAss Featuring Information Compression (CLAFIC)~\cite{clafic1967} (also known as Nearest SubSpace (NSS) classifier ~\cite{ChiNSS2012}; for the information contained in this name, we will adopt the usage of NSS throughout the paper). In this algorithm, each class is represented by a linear subspace and data instances are assigned to the nearest subspace. Instead of obtaining good representation of subspaces in NSS, the Learning Subspace Method (LSM)~\cite{Kohonen79} proposes to learn the subspaces based on good discrimination (see~\cite{oja1983subspace} for more variants and discussions). The simple idea of subspace classifiers has been extended to nonlinear versions in various ways; many have shown state-of-the-art performance (see~\cite{SS08,Cevikalp2008,kNS11} for example and Section~\ref{sec:discussions} for more details).  After the first subspace analysis of face images~\cite{Kirby90,Turk91}, classification approaches with subspace models have been used successfully in face recognition~\cite{Cappelli2001}, handwritten digit recognition~\cite{Laaksonen97thesis}, speech recognition~\cite{Kohonen80speech} as well as biological pattern recognition problems~\cite{Okun_proteinfold04}.

Although the design of subspace-based classification techniques has been actively explored, their theoretical justification is very under-studied. In this paper, we restrict our interests of justification to analyzing how well the classifiers can be generalized to new samples. By doing so, one can learn quantitatively how reliable the classification approaches are and can thus also guide the algorithm design accordingly. For this purpose, a functional (known as {\it risk function}) is used to measure the prediction quality of every classifier. More precisely, we assume $X$ and $Y$ being random variables; instances $\bx_i$ and $y_i$ are drawn independently from the distributions of $X$ and $Y$ respectively. For a classifier $f(x)$, its risk functional is defined as:
\bes
R(f) = \bbE_{(X,Y)} \mathbbm{1}(f(X)\neq Y)
\ees
Based on this, the {\it optimal Bayes rule} is defined to be the classifier whose risk functional is minimal. The Bayes rule is optimal in the sense that its expected loss (defined as 1 when the predicted class is not equal to the truth) is minimal. Note that, since the actual distribution of $(X,Y)$ is unknown, the Bayes rule is thus not available in reality. A natural desirable property of practical classifiers is having as small risk functional as possible. In this spirit, the property {\it consistency} is defined as the fact that the risk function converges to that of the optimal Bayes rule. In other words, classifier that is {\bf not} consistent produces larger misclassification errors on average than the best scenario, no matter how many data samples are available. Many classification algorithms, such as, KNN, SVM, LDA and some boosting methods~\cite{bachECCV2008,Stone1977,Steinwart02,LDAconsistency05,boostingConsistency2003,AdaBoostConsistency2007}, have been shown to be consistent under certain conditions.

In this paper, we study the consistency property of the Nearest SubSpace (NSS) classifier. We prove its strong consistency under certain conditions.
We also validate the performance of NSS through fruitful experiments, in comparison with other linear classifiers, LDA, FDA and SVM. These experiments demonstrate that NSS has very effective and comparable performance as its better known and more popular competitors. Since the classifiers under consideration are all simple and fundamental ones, they are not state-of-the-art. However, they are very important components of classification and such an experimental comparison completes the understanding of NSS. For our best knowledge, an experimental comparison like this (between NSS and other typical linear classifiers) has not been demonstrated yet. In the rest of the paper, we will begin with a description of the NSS algorithm (Section~\ref{sec:alg_thm}), followed by its consistency analysis (Section~\ref{sec:proof}) and experiments (Section~\ref{sec:experiments}).

\section{The NSS Algorithm and its Strong Consistency\label{sec:alg_thm}}
For most of the applications, it suffices to assume that $\cX \subset \cB(0,M) \subset \reals^D$ and $\cY=\{1,\cdots, K\}$, where $\cB(0,M)$ is the ball centered at the origin with radius $M$ and $D$ and $K$ are some positive integers. We will restrict ourselves to this case throughout the paper.
\subsection{The NSS Algorithm}
The NSS classifier assumes data lie on multiple affine subspaces, finds an estimate for these subspaces and assigns each instance to the nearest subspace. The following is a summary of the NSS algorithm.

\begin{algorithm}[htbp!]                      
\caption{Nearest Subspace (NSS) Classification}          
\label{alg}                           
\begin{algorithmic}                    
    \REQUIRE $\{(\x_i,y_i)\}_{i=1}^n \subset \cX \times \cY$ and $d$: intrinsic dimension, some positive integer and $d<D$.
    \ENSURE A function $f:\cX\rightarrow \cY$.
    \FOR{$k=1$ \TO $K$}
          \STATE
          \begin{eqnarray}
           \hat{\bu}_k &=& \frac{1}{n_k} \sum_{\bx_i \in C_k} \bx_i; \;C_k = \{\bx_i: \,y_i=k \}; \; n_k = |C_k| . \nonumber \\
           \hat{\bB}_k &=& \argmin_{\substack{\bB \in \reals^{D\times d}\\ \bB^T\bB = \bI_d}} \sum_{\bx_i \in C_k} \|(\bI-\bB\bB^T)(\bx_i-\hat{\bu}_k) \|^2.  \label{eq:B_k}
           \end{eqnarray}
    \ENDFOR
    \STATE      $ \hat{f}(\x)= \displaystyle \argmin_{k} \| (\I - \hat{\bB}_k \hat{\bB}_k^T) (\x - \hat{\bu}_k) \|^2  $.
\end{algorithmic}
\end{algorithm}

Note that the closed form solution to~\eqref{eq:B_k} is the Singular Value Decomposition (SVD) of the centered data matrix for the $k^{\text{th}}$ class; such a data matrix consists of $\begin{pmatrix}(\bx_{k_1}-\hat{\bu}_k),&\cdots,&(\bx_{k_{n_k}}-\hat{\bu}_k)\end{pmatrix}$ with $\bx_{k_1},\cdots,\bx_{k_{n_k}} \in C_k$.

\subsection{The Main Theorem}
As mentioned in Section~\ref{sec:intro}, a desirable property for classifiers is {\it consistency}. Denote $h_n$ to be any classification rule determined from $n$ samples $\{(\x_i,y_i)\}_{i=1}^n$, $f^*$ as the optimal Bayes rule, i.e., $f^*=\argmin_f R(f)$ and $R^*:=R(f^*)$ as its risk. Now we define strong consistency in the following sense.
\begin{mydef}[Strong Consistency]
A classification rule $h_n$ is said to be strongly consistent if
\bes
R(h_n) \to R^*  \;\; \text{a.s.}
\ees
\end{mydef}

Since the NSS classifer is also based on $n$ samples $\{(\x_i,y_i)\}_{i=1}^n$, from now on, we denote it as $\hat{f}_n$ for it for the rest of the paper. Then we obtain the following theorem for the NSS classifier described in Algorithm~\ref{alg}.

\begin{thm}
The NNS classifier $\hat{f}_n$ is strongly consistent, i.e., $R(\hat{f}_n) \to R^* \;\; \text{a.s.}$, when the following assumptions hold.\\
\indent (1) $(\x_1,y_1),...,(\x_n,y_n)$ are i.i.d.~samples of random variable $(X,Y)$; $X \in \R^D$ and $Y \in \{1,\dots,K\}$. \\
\indent (2) $\bbP(Y=i)=\frac{1}{K}$. \\
\indent (3) $X|Y=k \sim \mu_{L_k} \times \mu_{L_k^\perp}$; $L_k$ is the underlying $d$-dimensional subspace for the $k^{\text{th}}$  class; $\mu_{L_k}$ is a uniform measure on $L_k \cap \cB(\bu_k,M)$ (a bounded ball centered at $\bu_k$, the underlying center for the $k^{\text{th}}$  class); $\mu_{L_k^\perp}$ is a measure on $L_k^\perp$ decreasing exponentially w.r.t. the square distance from $L_k$;
\label{thm}
\end{thm}

This theorem reveals that the average prediction error of NSS converges to the optimal prediction error under certain conditions. It is a similar but slightly weaker result in contrast to that for LDA in~\cite{LDAconsistency05}, since the above condition (3) is stronger than that for LDA. Note that both results are about consistency for a class of distributions. On the other hand, the consistency results for KNN, SVM and some boosting methods are for all distributions, and thus are more general~\cite{Stone1977,Steinwart02,boostingConsistency2003,AdaBoostConsistency2007}.


\subsection{Discussions\label{sec:discussions}}
The NSS algorithm is a very simple and basic classification method, since it assumes linear structure in data. Linear models have their limitations, since the linearity constraint often is not satisfied in real data. However, they are important for the following reasons: (1) Linear classifiers are easy to compute and analyze. (2) They are a first order approximation for the true classifier. (3) They often have good interpretations, critical in many applications. (4) Linear models are the best that can be done when the available training data are limited. (5) Linear models are the foundation from which more complex models can be generalized (see~\cite{ElemStat01} for more discussion). Therefore, it is important to study this class of methods thoroughly, even if in practice they are no longer state-of-the-art. The computational complexity and the extensions of the NSS algorithm are further discussed as follows.

\paragraph{Complexity.} It is worth pointing out that the NSS algorithm is efficient. Assuming $D<n$, the computational complexity of the training process of NSS is $\mathcal{O}(KD^2(\bar{n}_k+2D))$ where $\bar{n}_k=\frac{1}{K}n_k$. LDA and FDA have similar complexity since they all require some eigen- or singular value decomposition operations. On the other hand, SVM requires $\mathcal{O}(n^2)$ to $\mathcal{O}(n^3)$ operations. Therefore, for large scale ($n$ large and $n\gg D$) problems, the computation of NSS is much faster than for linear SVM. In cases when the data is of large scale and some sensible results are needed quickly, NSS is a good choice. Section~\ref{sec:experiments} will provide more details of the performance in terms of both accuracy and speed of the algorithms.


\paragraph{Extensions.} The NSS method has been modified and extended through different methods: localization, the kernel trick and the hybrid model. The local subspace methods find, for the investigated data sample, their nearest neighbors in each class and attribute by their distances to the subspace spanned by these neighbors~\cite{SkarbekGI97,Laaksonen97thesis,HKNN01,Cevikalp2008,kNS11}. Due to the fact that only an inner product is needed in the NSS algorithm, it can be naturally extended by the kernel trick, where the original data are embedded into a higher dimensional space and subspace structures are learned there~\cite{Laaksonen97thesis,Tsuda99,Bala99,Zhao99,MaedaM02}; these two techniques are combined in~\cite{Zou00}. Another direction is to represent each class by multiple subspaces~\cite{Laaksonen97thesis,SS08,kNS11}, where~\cite{SS08} also uses a more general metric than the Euclidian distance. All of these extended techniques define nonlinear decision boundaries and the recent works~\cite{SS08,Cevikalp2008,kNS11} have shown their state-of-the-art performance.

%

\section{Proof of Theorem~\ref{thm}\label{sec:proof}}
In this section, we give a complete proof of Theorem~\ref{thm} following~\cite{LDAconsistency05}.
\subsection{Notations}
We first describe the problem in detail and prepare to prove the theorem. Consider a classification problem, where the goal is to assign an individual instance to one of $K$ classes, given $n$ observations of $(X,Y)$. To do this, the space $\reals^D$ is partitioned into subsets $H_1,\dots,H_K$ such that, for $k=1,\dots,K$, the individual instance is classified to be in group $k$ when $X \in H_k$. This procedure generates a discriminant rule as a mapping $f:\R^D \to \{1,\dots,K\}$ that takes the value $f(X)=k$ whenever the individual is assigned to the $k^{\text th}$ group, and this can be written as $f(X) =  \sum_{k=1}^K k \mathbbm{1}_{H_k}(X)$, where $\mathbbm{1}_{H_k}(X)$ is the indicator function of the subset $H_k$.

Let $Y$ be the discrete random variable (class index or group label) which represents the true membership of the individual under study. Denote the class prior probabilities $\pi_k = \bbP[Y=k]>0$, $\sum_{k=1}^K \pi_k = 1$ and $k=1,\dots,K$. Furthermore, assume there exist density functions $g_k(X)$ such that $\bbP[X \in \cA | Y=k] = \int_\cA g_k(X) d X,\;k=1,\dots,K$ for $\cA$, a subset of $\reals^D$.

Given $(X,Y)$, the rule $f(X) = \sum_{k=1}^K kI_{H_k}(X)$ is in error when $f(X) \neq Y$ and its probability of misclassification is computed as:
\begin{eqnarray}
 R(f)& = &\bbE_{(X,Y)} \mathbbm{1}(f(X)\neq Y) = \bbP[f(X) \neq Y]=1-\bbP[f(X) = Y]  \nonumber \\
 &=& 1-\sum_{k=1}^K \bbP[X \in H_k,\,Y=k] \nonumber \\
 &=& 1-\sum_{k=1}^K \bbP[Y=k]\bbP[X \in H_k|Y=k] \nonumber \\
 &=& 1-\sum_{k=1}^K \pi_k \int_{H_k} g_k(X)dX .
 \label{eq:Lfx}
 \end{eqnarray}

The rule $f^* = \sum_{k=1}^K k\mathbbm{1}_{H^*_k}(X)$ that minimizes~\eqref{eq:Lfx}, or the Bayes rule, is given by the partition
\bes
H_k^* = [X:\; \pi_k g_k(X)=\max_{\mathbbm{1}\leq j \leq K}\pi_j g_j(X) ], \;\; k=1,\dots,K.
\ees
Then the corresponding optimal error is:
\bes
R^* =  R[f^*(X)] = 1-\sum_{k=1}^K \pi_k \int_{H^*_k} g_k(X)dX .
\ees

In general, both $\pi_k$ and $g_k$ are unknown, so rules used in practice are sample based rules of the form $\hat{f}_n(X) = \sum_{k=1}^K kI_{\hat{H}_{k,n}}(X)$, where the subsets $\hat{H}_{k,n}$ depend on the data set $\Omega_n = \{(\bx_i,y_i)\}_{i=1}^n$ formed by $n$~i.i.d. observations from $(X,Y)$. The appropriate measure of error of a sample rule $\hat{f}_n(X)$ is $R_n=\bbP[\hat{f}_n(X) \neq Y]$.


%
%
\subsection{Proof of Theorem~\ref{thm}}
We will first prove a useful lemma which gives a bound for $R_n-R^*$.
\begin{lemma}
\label{lemma}
Assume $\pi_k=\frac{1}{K}$ and let $\hat{g}_{k,n}(X)$ be an estimate of $g_k(X)$ from $\Omega_n$ , for $k=1,\cdots,K$. Let $\hat{f}_n(X)$ be the classifier derived from $\hat{g}_{k,n}(X)$, i.e., $\hat{f}_n(X)=\argmax_k\hat{g}_{k,n}(X)$. Then
\bes
0\leq R_n - R^* \leq \frac{1}{K} \sum_{k=1}^K \int |g_k(X)-\hat{g}_{k,n}(X)| dX.
\ees
\end{lemma}
\begin{proof}
Since $\pi_k=\frac{1}{K}$, we have $H_k^* = [X:\; g_k(X)=\max_{1\leq j \leq K}g_j(X)$. Thus,
\begin{eqnarray}
R^*&=&1-\frac{1}{K}\sum_{k=1}^K \int_{H_k^*}g_k(X)dX = 1 - \frac{1}{K}\int \max_k g_k(X) dX \nonumber \\
&\leq& 1 - \frac{1}{K}\int g_{\hat{f}_n}(X) dX = R_n \nonumber
\end{eqnarray}
On the other hand,
\begin{eqnarray}
R_n - R^* &=& \frac{1}{K} \int (\max_k g_k(X) - g_{\hat{f}_n}(X)) dX \nonumber \\
&=&\frac{1}{K} \int (\max_k g_k(X) - \hat{g}_{\hat{f}_n,n}(X)) dX + \frac{1}{K}\int (\hat{g}_{\hat{f}_n,n}(X) - g_{\hat{f}_n,n}(X) ) dX \nonumber \\
&=& \frac{1}{K}\int (\max_k g_k(X) - \max_k \hat{g}_{k,n}(X)) dX + \frac{1}{K}\int (\hat{g}_{\hat{f}_n,n}(X) - g_{\hat{f}_n,n}(X) ) dX \nonumber \\
&\leq& \frac{1}{K}\sum_{k=1}^K \int |g_k(X) - \hat{g}_{k,n}(X)| dX. \nonumber
\end{eqnarray}

\end{proof}

A similar result of Lemma~\ref{lemma} can be found in the Theorem 1 in~\cite{Devroye1985}~(p.~254). Now we prove the main theorem of our paper.

\begin{proof}[Proof of Theorem~\ref{thm}]
Due to condition (2), we have
\bes
H^*_k = [X: g_k(X) = \max_{1\leq j \leq K} g_j(X)] \label{eq:H_star}.
\ees

On the other hand, based on the assumption (3), the density functions can be written as
\begin{align*}
&g_k(X) = C(d) \beta \exp{(-\alpha t )}, \\
&t = (X-\bu_k)^T(I-\bP_k)(X-\bu_k)
\end{align*}
for some $\alpha>0, \,\beta$ independent of $t$, constant $C(d)$ and $\bP_k=\bB_k\bB_k^T$ with $\bB_k$ being the orthonormal basis for $L_k$.

Then the classifier generated by the Algorithm~\ref{alg} can be written as:
\bes
\hat{f}_n(X) = \sum_{k=1}^K kI_{\hat{H}_{k,n}}
\label{eq:fn_hat}
\ees
with the following notation:
\begin{align*}
&\hat{\bP}_k = \hat{\bB}_k\hat{\bB}_k^T  \\
&\hat{g}_{k,n}(x) = C(d) \beta \exp{(-\alpha (X-\hat{\bu}_k)^T(I-\hat{\bP}_k)(X-\hat{\bu}_k))} \\
&\hat{H}_{k,n} = [X: \hat{g}_{k,n}(X) = \max_{1\leq j \leq K} \hat{g}_{j,n}(X)] \label{H_hat}
\end{align*}

Thus the NSS classifier can be considered as a plug-in version of the Bayes rule. By Lemma~\ref{lemma}, the difference $R_n-R^*$ can be bounded in the form
\bes
0 \leq R_n-R^* \leq \frac{1}{K} \sum_{k=1}^K \int_{\R^D} |g_k(X) - \hat{g}_{k,n}(X)| d X \label{eq:RnRstar}
\ees

For each fixed $1\leq k \leq K$, we have
\bes
0 \leq \int_{\R^D} |g_k(X) - \hat{g}_{k,n}(X)| d X \leq \int_{\R^D} g_k(X) + \hat{g}_{k,n}(X) d X < \infty
\ees

Therefore, it suffices to show that $\hat{g}_{k,n} \to \hat{g}_k$~{\it a.s} and due to the continuity of $g(\cdot)$, to show $\hat{\bu}_k \to \bu_k$ and $\hat{\bP}_k \to \bP_k$~{\it a.s}. The fact that $\hat{\bu}_k$ and $\hat{\bP}_k$ are the maximum-likelihood estimations (MLE) of $\bu_k$ and $\bP_k$ completes the proof.
\end{proof}

\section{Experiments \label{sec:experiments}\label{sec:experiments}}
In this section, we evaluate the performance of the NSS algorithm through various experiments and compare it with LDA, FDA and linear SVM. The purpose of demonstrating these results is two-fold. First, it is to show that, as a simple and basic method, the NSS algorithm can obtain very useful results and is comparable to its competitors. Second, it serves as a complementary perspective to the theoretical portion of the paper. The reason why we include LDA, FDA and linear SVM in the comparison is because they are similar to NSS. Note that the objective is not to prove NSS is state-of-the-art. On the other hand, the significance of studying this method has been fully discussed in Section~\ref{sec:discussions}.

\subsection{Data}
We test the classification methods on two simulated data sets and five real data sets. In the following, we will give a brief description for each of them and a summary of size, dimension and the number of classes can be found in Table~\ref{tab:data}.

\paragraph{Mixture Gaussian.} Data samples are generated from $K=3$ Gaussian distributions in $\R^3$ with means $\mu_1=(1,2,3)^T,\,\mu_2=(-1,-2,-3)^T,\,\mu_3=(-1,2,-3)^T$ and variances $ \Sigma_1=\begin{pmatrix}
3 & 0.2 & 0.1\\
0.2 & 2 & 0.2\\
0.1 & 0.2 & 2 \end{pmatrix},\,
\Sigma_2= \begin{pmatrix}
2 & 0 & 0\\
0 & 1 & 0\\
0 & 0 & 1 \end{pmatrix},\,
\Sigma_3= \begin{pmatrix}
2 & 0 & 0\\
0 & 2 & 0\\
0 & 0 & 3 \end{pmatrix}$.
The total number of samples is 1200; 400 in each class.

\paragraph{Multiple Subspaces.} For the multiple subspaces experiment, data are generated uniformly from 3 2-dimensional linear subspaces (bounded in a unit disk) in $\R^{50}$. The angles between the subspaces are at least $\frac{\pi}{8}$. Gaussian noise with $0$ mean and $0.05$ standard deviation is added. Again, 1200 samples are generated in total with 400 in each class.

\paragraph{Wine.} Wine recognition data are the results of a chemical analysis of wines grown in the same region in Italy but derived from three different cultivars. The goal is to determine the types of wines from the quantities of 13 constituents found in them. The data were first collected in~\cite{wine} and now can be found in the UCI machine learning repository.

\paragraph{DNA.} We use the Statlog version of the primate splice-junction DNA data set (found in~\cite{dna}). The problem is to recognize, given a sequence DNA, the boundaries between exons (the parts of the DNA sequence retained after splicing) and introns (the parts of the DNA sequence that are spliced out). The features are binary variables representing nucleotides in the DNA sequence. Three classes are ``intron to exon" boundary, ``exon to intron" boundary and neither. This data set has three subsets, training, evaluation and testing. All of them are used in our experiments.

\paragraph{USPS.} USPS~\cite{usps} is a database of scanning images of handwritten digits from US Postal Services envelopes. The goal is to recognize digits given their $16\times 16$ grayscale images. Both the training and testing sets are used in our experiments.

\paragraph{Vehicle.} This Vehicle data set~\cite{vehicle} collects signals obtained by both acoustic and seismic sensors and the goal is to classify vehicle types from the original data. It has two subsets, training and testing, and we use both of them in our experiments.

\paragraph{News20.} The 20 Newsgroups data set is a collection of approximately 20,000 newsgroup documents, partitioned (nearly) evenly across 20 different newsgroups. The problem imposed here is to recognize the newsgroups from texts. This data was originally collected in~\cite{news20}. Due to the very large scale of the data, we use only the testing set in our experiment for simplicity.

\setlength{\abovecaptionskip}{4pt}   
\setlength{\belowcaptionskip}{4pt}
\begin{table}[htbp!]
\centering
\caption{A summary of the data sets} \label{tab:data}
\begin{tabular}{|@{}c@{}|@{}c@{}|@{}c@{}|@{}c@{}|@{}c@{}|}
\hline
\multirow{2}{*}{data}& data size &\multirow{2}{*}{\# of classes} & ambient dimension&reduced\\
& (\# of samples) & &( \# of features) & dimension \\
\hline
Mixture Gaussian &  1200 & 3 & 3 & NA\\
Multiple Subspaces &  1200 & 3 & 50 & NA\\

\hline
Wine&178&3&13&NA \\
Vehicle&98,528&3&100&NA\\
DNA&3,186&3&180&NA\\
USPS&9,298&10&256&38\\
News20&3,993&20&62,060&1000\\
\hline
\end{tabular}
\end{table}

\subsection{Implementation Details}
Real data used in our experiments are originally from the UCI machine learning repository, Statlog and other collections. We download them from~\cite{dna}, where data samples have been scaled linearly to be within [0, 1] or [-1, 1]. For the data sets USPS and News20, the ambient dimension is reduced by Principal Component Analysis to be at most 1000 and such that 95\% variance is explained. The reduced dimension for the USPS and News20 data sets is shown in the last column of Table~\ref{tab:data}. For NSS,  the intrinsic dimension of the subspaces is determined by 10-fold cross validation.

The classification experiments are carried out in Matlab. We use the default function {\it classify} of the Statistics toolbox for LDA. For multiclass FDA and SVM (see~\cite{multiFDA} and~\cite{multiSVM_CS}), we use implementations from~\cite{multiFDA_matlab} and~\cite{MSVMpack}. The NSS can be simply realized and the version we use can be found from the author's homepage \url{http://www.math.duke.edu/~yiwang/}.

\subsection{Results}
For each data set, we randomly split it into two subsets, each with 80\% and 20\% of the data, and use the former as the training set and the latter as the testing set. All experiments are repeated 200 times for the simulated data sets and 10 times for the real data sets, including the random generation (for the simulated data sets) and the random splitting processes. The mean and standard deviation of the accuracy of all methods under investigation are reported in Table~\ref{tab:results}, while the running time for the training process is recorded in Table~\ref{tab:time}.

\setlength{\abovecaptionskip}{4pt}   
\setlength{\belowcaptionskip}{4pt}
\begin{table}[htbp!]
\centering
\caption{A summary of classification results: mean accuracy $\pm$ standard deviation (\%)} \label{tab:results}
\begin{tabular}{|c|c|c|c|c|}
\hline
\multirow{2}{*}{Data} & \multicolumn{4}{c|}{Methods} \\
\cline{2-5}
 & NSS & LDA & FDA & SVM \\
\hline
Gaussian &  88.11 $\pm$ 4.47 & {\bf 95.12} $\pm$ 1.58 & 81.43 $\pm$ 3.74 & 94.93 $\pm$ 1.58  \\
Subspace & {\bf 99.16} $\pm$ 0.51 & 34.97 $\pm$ 2.73 & 33.74 $\pm$ 2.26 & 46.57 $\pm$ 3.23  \\

\hline
Wine&94.29 $\pm$ 3.81 & {\bf 98.57} $\pm$ 2.02 & 92.57 $\pm$ 3.86 & 96.29 $\pm$ 1.93 \\
Vehicle&74.23 $\pm$ 0.34 & {\bf 80.15} $\pm$ 0.22 & 75.11 $\pm$ 0.33 & \text{NA} \\
DNA&90.28 $\pm$ 1.05 & {\bf 93.23} $\pm$ 1.02 & 78.59 $\pm$ 2.07 & 91.11 $\pm$ 0.94 \\
USPS&{\bf 96.54} $\pm$ 0.39 & 91.19 $\pm$ 0.58 & 48.58 $\pm$ 1.01 & 94.02 $\pm$ 0.48 \\
News20& 75.18 $\pm$ 1.60 & 35.55 $\pm$ 1.72 & 9.50 $\pm$ 1.39 & {\bf 75.54} $\pm$ 1.26 \\
\hline
\end{tabular}
\end{table}

\setlength{\abovecaptionskip}{4pt}   
\setlength{\belowcaptionskip}{4pt}
\begin{table}[htbp!]
\centering
\caption{A summary of running time on real data sets (seconds)} \label{tab:time}
\begin{tabular}{|c|c|c|c|c|}
\hline
\multirow{2}{*}{Data} & \multicolumn{4}{c|}{Methods} \\
\cline{2-5}
 & NSS & LDA & FDA & SVM \\
\hline
Wine& $8.175\times 10^{-4}$ & 0.012 & 0.019 & 0.654\\
Vehicle&1.064&1.206&3.449&\text{NA}\\
DNA&0.042 & 0.037&0.212&73.685 \\
USPS&0.035&0.052&0.148&935.296 \\
News20& 1.089&1.018&12.841&168.209 \\
\hline
\end{tabular}
\end{table}

From the above results, we know that the NSS algorithm can obtain results comparable to its better known competitors LDA and SVM, for a broad range of classification problems. Meanwhile, the computation is very fast, roughly the same order as FDA and LDA, but significantly faster than SVM, especially for large scale problems. Additionally, LDA requires that the covariance matrix is positive definite, which is not satisfied in some high dimensional data sets. This is another reason why we reduce the ambient dimension for the USPS and News20 datasets. However, NSS does not have this restriction.

\section{Conclusion}
In this paper, we reviewed a simple classification algorithm (NSS) based on the model of multiple subspaces. We proved its strong consistency under certain conditions, which means that under these conditions, the prediction error of NSS on average converges to that of the optimal classifier. Finally, we evaluated NSS on various data sets and compared it with its competitors. Results showed that NSS can obtain very useful results efficiently, especially for large scale data sets.

By studying the consistency property of NSS, we are inspired to further explore subspace-based classification methods along the following directions in the future. First, NSS finds a good estimation for the underlying subspace models by minimizing the sum of squares of fitting errors. However, for the purpose of classification, it is more helpful to obtain models which can ``separate" or ``discriminate" classes. Therefore, in order to improve the classification performance, some separation measure can be taken into account. In fact, an advanced supervised learning method based on multiple subspaces has been proposed~\cite{SS08}. It would be fruitful to analyze this method or other variants theoretically.

Moreover, a general way to find a good classifier is to minimize an empirical risk function, which is typically defined as $R_{\text{emp}}(f) = \sum_{i=1}^n \mathbbm{1}(f(\bx_i)\neq y_i)$. This idea can be combined with the multiple subspaces model. Similar approaches to that in~\cite{Vapnik1999} can be applied to analyze its consistency.

\section*{Acknowledgements}
We thank Dr.~Gilad Lerman for motivating the author to work on the problem and for valuable discussions and Dr.~Mauro Maggioni for inspiring comments and suggestions.

\bibliographystyle{plain}
\bibliography{hlm_classifier}
\end{document}